\documentclass[12pt,reqno]{amsart}
\usepackage{latexsym, amsmath, amscd, amssymb, amsthm}
\usepackage{bm}
\usepackage{mathrsfs}   
\usepackage{mathtools}
\usepackage{dsfont}
\usepackage{xspace}
\usepackage{natbib}
\usepackage{setspace}
\usepackage{enumerate}
\usepackage{verbatim}
\usepackage[english]{babel}
\usepackage{epsfig}
\usepackage[ruled,vlined]{algorithm2e}
\usepackage{subcaption}
\newtheorem{lemma}{Lemma}[section]
\newtheorem{theorem}[lemma]{Theorem}

\theoremstyle{definition}
\newtheorem{definition}{Definition}

\newtheorem*{remark}{Remark}

\theoremstyle{remark}

\DeclarePairedDelimiter\ceil{\lceil}{\rceil}

\newcommand{\argmax}{\mathop{\mathrm{argmax}}\limits}

\newcommand{\ud}{\mathrm{d}}

\begin{document}
\title[Non-stationary multi-armed bandits with fairness constraints]
{A Regret bound for non-stationary multi-armed bandits with fairness constraints}
\author[Shaarad A.R and A. Dukkipati]{Shaarad A.R, Ambedkar Dukkipati}
\address{Dept. of Computer Science \& Automation\\Indian Institute of Science, Bangalore}
\email{\{rangaa,ad\}@iisc.ac.in}

\begin{abstract}
The multi-armed bandits' framework is the most common platform to study strategies for sequential decision-making problems. Recently, the notion of fairness has attracted a lot of attention in the machine learning community. One can impose the fairness condition that at any given point of time, even during the learning phase, a poorly performing candidate should not be preferred over a better candidate. This fairness constraint is known to be one of the most stringent and has been studied in the stochastic multi-armed bandits framework in a stationary setting for which regret bounds have been established. The main aim of this paper is to study this problem in a non-stationary setting. 
We present a new algorithm called Fair Upper Confidence Bound with Exploration (\textsc{Fair-UCB}e) algorithm for solving a slowly varying stochastic $k$-armed bandit problem. With this we present two results: (i) \textsc{Fair-UCB}e indeed satisfies the above mentioned fairness condition, and (ii) it achieves a regret bound of $O\left(k^{\frac{3}{2}} T^{1 - \frac{\alpha}{2}} \sqrt{\log T}\right)$, for some suitable $\alpha \in (0, 1)$, where $T$ is the time horizon. This is the first fair algorithm with a sublinear regret bound applicable to non-stationary bandits to the best of our knowledge. We show that the performance of our algorithm in the non-stationary case approaches that of its stationary counterpart as the variation in the environment tends to zero.
\end{abstract}

\clearpage\maketitle
\thispagestyle{empty}
\section{Introduction}
Multi-armed bandits and other related frameworks for studying sequential decision-making problems have been found to be useful in a wide variety of practical applications. For example, bandit formulations have been used in healthcare for modelling treatment allocation~\citep{villar2015multi, pmlr-v85-durand18a}, studying influence in social networks~\citep{wu2019factorization, wen2017online}, recommendation systems~\citep{zhou2017large} etc. The present paper deals with incorporating fairness conditions in multi-armed bandit problems where the underlying environment is non-stationary.

How can a bandit algorithm be `unfair'? In the classic stochastic $k$-armed bandit problem, at each time step, an agent has to choose one out of $k$ arms. When an arm is chosen, the agent receives a real-valued reward sampled from a probability distribution corresponding to the chosen arm. The goal of the agent is to maximize the expected reward obtained over some time horizon $T$. For this, the  learning algorithm has to initially try out each arm to get an idea of its corresponding reward distribution. This is referred to as the exploration phase. Once the agent gathers enough information about the reward distribution of each arm, it can then make an informed decision and choose among the arms in such a way as to maximize the rewards obtained. The performance of the agent is usually measured with a notion of regret. This is defined as the expected difference in the rewards obtained if the agent follows the optimal policy of choosing the best arm versus the policy actually followed by the agent. 

In some socially relevant practical problems of sequential decision making, 
judging an algorithm solely based on regret may not be enough. The reason being that regret only provides the picture of expected returns and does not deal with behavior or in what way decisions are taken, especially during the initial learning phase. Is it OK to be `unfair' to a better candidate just because the learning agent or algorithm is still trying to learn?
For example, selecting a poorly performing arm for a constant number of times does not affect the asymptotic regret achieved by the algorithm, but this behavior would still be unfair to better performing arms, which have been ignored either deliberately or due to carelessness during the learning phase.

To deal with such issues, one can enforce some well-defined fairness constraints on a learning algorithm in addition to the goal of minimizing regret. Various definitions of fairness motivated by real-world applications have been studied in the context of stationary multi-armed bandit problems~\citep{joseph2016fairness, li2019combinatorial}. Intuitively, these fairness conditions insist that any algorithm solving the bandit problem should consider all the arms and their rewards and should be `fair' when selecting the arms. This might be an important requirement in many real-world applications, especially those involving humans. The learning algorithms should be designed such that they do not give rise to decisions that unduly discriminate between different individuals or groups of people. 

One notion of fairness is what could be considered equality of opportunity, in which arms with similar reward distributions are given a similar opportunity of being chosen at any given time. For instance, with high probability (at least $1 - \delta$), at each time step, an arm can be assigned a higher probability of being chosen than another arm only if the expected reward of the former is strictly greater than that of the latter~\citep{joseph2016fairness}. This notion of $\delta$-fairness is what is considered in this paper when referring to the fairness of the proposed algorithm.


Most standard solutions to stochastic multi-armed bandit problems assume that the rewards are generated independently from reward distributions of arms. These distributions remain fixed over the entire time horizon $T$. However, in many practical problems, the underlying environment cannot be expected to remain fixed.
This amounts to multi-armed bandit problems where the reward distributions may change at each time step. This requires developing learning algorithms that should be able to cope with different kinds of changes in the environment. For example, a bandit algorithm for a recommendation system should be able to handle a change in user preferences over time~\citep{KDD-2016-ZengWML}.

If there are no statistical assumptions or constraints on the rewards corresponding to any arm at any time step, the problem becomes what is referred to as the adversarial bandit problem~\citep{auer1995gambling}. This problem is difficult to solve under the classic notion of regret since any information obtained about the reward distribution of an arm at a certain time becomes useless in the next time step, which means that any arm which is optimal at a one-time step need not be optimal later. However, this setting can still be studied and solved with respect to a weaker notion of regret. Here, the regret of an algorithm is measured against a policy that is allowed to choose the same arm during all time steps.

Since the adversarial setting is too general for obtaining good results with respect to the standard form of regret, other variations of non-stationary bandit problems have been extensively studied, which constrain how the reward distributions of the arms change as time passes. One possible constraint is bound on the absolute change in the expected rewards of all arms at each time step~\citep{wei2018abruptly}. In this paper, we consider a variant of this slowly varying environment.

Since existing fair algorithms assume a stationary environment, their fairness guarantees do not hold when the stationarity assumptions are no longer true. Hence, modifying these algorithms to respect the fairness constraints in a non-stationary environment is non-trivial. In this work, we address the problem of satisfying fairness constraints in a slowly varying non-stationary environment.

\subsection*{Contributions}
In the literature, fair bandit algorithms have not been studied in the non-stationary setting. The main contribution of this paper is a fair UCB-like algorithm for solving a non-stationary stochastic multi-armed bandit problem. The environment considered is a slowly varying environment.
We prove that the proposed algorithm is $\delta$-fair (the fairness condition considered in~\citep{joseph2016fairness})
and achieves a regret of order $O  \left(k^{\frac{3}{2}}  T^{1 - \frac{\alpha}{2}} \sqrt{\log (cT^{1 + \frac{\alpha}{4}})} \right)$ for some $\alpha \in (0, 1)$ and constant $c \in \mathbb{R}_+$. As the non-stationarity of the environment considered is reduced, this regret bound approaches that achieved by a fair algorithm in the stationary setting, up to logarithmic factors.


\section{Preliminaries}
\label{prelims}
Consider a bandit with $k$ arms and a time horizon $T$. At time $t \in [T]$, let the reward  distribution of arm $i \in [k]$ be $\mathcal{P}_i^t$ on $[0,1]$, with mean $\mu_i^t$. Here, $[T]$ denotes the set $\{ 1, \ldots, T \}$ and similarly $[k]$ denotes the set $\{ 1, \ldots, k \}$. Given history $h^t \in ([k] \times [0,1])^{t-1}$ of arms chosen and rewards obtained till time $t-1$, the agent chooses an arm at time $t$ by sampling an arm from the probability distribution $\pi_{.|h^t}^t$, with the probability of choosing arm $i$ being $\pi_{i|h^t}^t$. Let $i^t$ be the arm chosen at time $t$, i.e, $i^t \sim \pi_{.|h^t}^t$. Note that in the stationary case, the reward distribution and the mean will remain constant, that is $\mathcal{P}_{i}^{t} = \mathcal{P}_{i}$ and $\mu_i^t = \mu_i$, for all $t \in [T]$, $i \in [k]$.

In this paper, we consider multi-armed bandits in a non-stationary setting, and hence, we assume that the means of the rewards distributions change as time progresses.
Our assumption can be stated as follows. 
We assume that there exists known parameter $\kappa \in \mathbb{R}_+$ such that for all $t < T$, and all arms $i \in [k]$, $| \mu_i^{t+1} - \mu_i^t | < T^{-\kappa}$, where $\mu_i^{t}$ and $\mu_i^{t+1}$ are the means of the reward distribution of arm $i$ at times $t$ and $t+1$ respectively. In other words, $\kappa$ controls how much the mean of the reward distribution of an arm is allowed to change at each time step. It is to be noted that the change in the mean depends only on the horizon $T$ and not the current time step $t$.

In this paper, we consider the notion of $\delta$-fairness that has been introduced in ~\cite{joseph2016fairness}. 
The intuition behind this definition of fairness is that at each time step, with a high probability of $1-\delta$, arms with similar reward distributions should have a similar chance of being selected. In other words, at any point in time, for any pair of arms, the learning algorithm should give preference to one of the arms over the other, only if it is `reasonably' certain that its expected reward is strictly greater than that of the other. This can be stated as follows.
\begin{definition}[$\delta$-Fairness] \citep{joseph2016fairness}
A multi armed bandit algorithm is said to be $\delta$-fair if, with probability atleast $1 - \delta$, $\forall i, j \in [k], t \in [T]$,
\begin{displaymath}
\pi_{i|h^t}^t > \pi_{j|h^t}^t \text{ only if } \mu_i^t > \mu_j^t,
\end{displaymath}
where $\pi_{i|h^t}^t$ and $\pi_{j|h^t}^t$ denote the probability assigned by the algorithm to choose arm $i$ and $j$ respectively at time $t$ given history $h^t$ of arms chosen and rewards obtained till time $t-1$, and $\mu_i^t$ and $\mu_j^t$ are the means of the rewards distributions of arms $i$ and $j$ respectively at time $t$.
\label{def_fairness}
\end{definition}
The dynamic regret achieved by a bandit algorithm is defined as
\begin{displaymath}
R(T) = \sum_{t=1}^T \max_{i \in [k]} \mu_i^t - \mathsf{E} \left[ \sum_{t=1}^T \mu_{i^t}^t \right],
\end{displaymath}
where $\mu_i^t$ is the mean of the reward distribution of arm $i$ at time $t$, and $i^t \in [k]$ is the arm selected by the algorithm at time $t$. Using the dynamic regret as defined above, the performance of a bandit algorithm is measured by comparing the expected reward of the arm selected at each time step against the expected reward of the optimal arm at that time step, taking into account the fact that the optimal arm changes with time. This is in contrast to the static regret considered in stationary and adversarial settings. In these settings, the performance of a bandit algorithm is measured against a single fixed optimal arm, which is the arm that gives the highest total expected reward over the entire time horizon $T$, when chosen at every single time step. Thus, dynamic regret is a stronger performance criterion than static regret.

\section{\textsc{Fair-UCB}e Algorithm}
\label{sec_algorithm}
Now we present some analysis that leads to our proposed algorithm.
For satisfying the fairness constraint as given by Definition \ref{def_fairness}, an arm should be preferentially chosen only if it is known that, with high probability, that arm indeed gives a strictly greater expected reward. To estimate this with high probability, confidence intervals are constructed for each arm, similar to the Upper Confidence Bound (UCB1) algorithm~\citep{AuerBianchiFischer:2002:FinitetimeAnalysisOfTheMultiarmedBanditProblem}. 
However, instead of choosing the arm with the highest upper confidence bound, arms with high estimates of expected rewards are chosen in such a way that fairness is maintained, as described below.

Let $[a_i^t, b_i^t]$ be the confidence interval of arm $i$ at time $t$. Suppose it has been proved that with probability atleast $1 - \delta$, $a_i^t \leq \mu_i^t \leq b_i^t$, for all $i \in [k]$ and $t \in [T]$. At each time $t$, to minimize the regret, UCB1 deterministically chooses the arm with the highest upper confidence bound $b_i^t$, say $i^* \in [k]$.
Fairness demands that $i^*$ be chosen only if $\mu_{i^*} > \mu_j$ for all $j \neq i^*$. However, if there exists an arm $j \neq i^*$ such that $b_j > a_{i^*}$, their confidence intervals overlap. There is no guarantee that the expected mean of arm $i^*$ is greater than that of $j$, forcing any fair agent to assign both these arms an equal probability of being chosen.

Now, with the above constraint, at time $t$, regret minimization requires the agent to choose arms $i^*$ and $j$ with probability $1/2$ each and ignore all the other arms. However, this is fair only if all the other arms have expected rewards strictly less than those of $i^*$ and $j$, which is not true if the confidence interval of some other arm $i$ overlaps with either of those of $i^*$ or $j$. Arm $i$ should also be assigned an equal probability of being chosen as $i^*$ and $j$, and the same argument extends to other arms whose confidence intervals overlap with that of $i$ and so on.

Let $B_t$ be the set of arms to be chosen at time $t$ with equal nonzero probability and let arms in $[k] \setminus B_t$ be ignored. For optimality, $B_t$ should contain $i^*$. For fairness, any arm whose confidence interval overlaps with that of any arm in $B_t$ should be added and this process should be repeated until no other arm can be added to $B_t$. $B_t$ is called the active set of arms~\citep{joseph2016fairness} and at each time step, an arm is chosen uniformly from the active set.
\begin{definition}[Active set]
Let $[a_i^t, b_i^t] \subseteq [0,1]$ be the confidence interval associated with each arm $i \in [k]$ of the bandit at time $t$. The active set $B_t \subseteq [k]$ is defined recursively as the set satisfying the following properties:
\begin{enumerate}[(i)]
    \item If $i^{*t} \in \argmax_i b_i^t$, then $i^{*t} \in B_t$.
    \item If $b_j^t \geq a_i^t$ for some $i \in B_t$, then $j \in B_t$.
\end{enumerate}
\end{definition}
Intuitively, the active set of arms is the set of arms whose confidence intervals are chained via overlap to the confidence interval with the greatest upper confidence bound. For the algorithm to be fair, each arm in the active set should be assigned an equal probability of being selected.

Due to non-stationarity, as time progresses, older samples become less indicative of the current reward distribution. Therefore, at each time step, we choose only the latest $\frac{t^\alpha}{k}$ samples of each arm to estimate the expected reward and construct the confidence interval, for suitably chosen $\alpha \in (0, 1)$. This progressive increase in the number of samples considered is similar to the use of a progressively increasing sliding window by \cite{wei2018abruptly}. Now, as time progresses, due to the increased number of samples obtained, the confidence intervals shrink, and the active set becomes small. If the arms that are not in the active set are ignored and not sampled for a long time, due to the nonstationarity of the environment, their expected rewards can change such that they fall into the confidence intervals of arms that are in the active set. 

So, to ensure that the learning algorithm does not remain oblivious to the reward distributions of inactive arms, we propose, with some fixed probability at each time step, to choose uniformly from \textit{all} arms. This exploration probability is chosen to be $T^{\frac{-\alpha}{2}}$, for suitable $\alpha \in (0, 1)$ to be specified. Due to this fixed exploration probability at each step, we refer to our proposed algorithm as \textsc{Fair}-UCB with exploration or \textsc{Fair}-UCBe. The overall steps involved are listed in Algorithm \ref{alg}. We present two results in this regard. First, we show that the proposed algorithm is indeed $\delta$-fair. Then, we establish an upper bound for the regret.

\begin{algorithm}
\DontPrintSemicolon
\SetKwInOut{Input}{Given}
\SetKwComment{Comment}{// }{}
\caption{Fair UCB with Exploration}
	\Input {Horizon $T$, drift parameter $\kappa \in \mathbb{R}_+$}
	Choose $\epsilon > \frac{1}{\log T} \log \bigg( \frac{1}{2\log (\frac{18}{11})} \log T \bigg)$, $\alpha < \min \{ 2 - \sqrt{2 \epsilon + 1}, \frac{1}{2}(\kappa - \epsilon), 1 \} $\;
  $\delta_2 \leftarrow T^{\frac{-\alpha}{2}}$\;
  \Comment{\small{Initialize sample sequences}}
  $S_i \leftarrow \phi$ for $i \in [k]$\;
	\ForEach {$t$ in $1, 2, \ldots, T$} {
	    \Comment{\small{Construct confidence intervals for each arm}}
		\ForEach {$i$ in $1, 2, \ldots, k$}  {
			\Comment{\small{Select the latest samples for estimation}}
			$\tau_i^t \leftarrow \min \{ |S_i|, \ceil{\frac{1}{k} t^\alpha} \}$\;
			$S'_i \leftarrow $ last $\tau_i^t$ elements of $S_i$\; 
			$\mu'_i \leftarrow $ mean($S'_i$)\;
			\Comment{\small{Determine width of confidence interval}}
			$c_i \leftarrow \sqrt{\frac{1}{2 \tau_i} \log(\frac{k \pi^2 t^2}{3 \delta_2})} + kT^{\frac{\alpha}{2} + \epsilon + 3}\frac{(\tau_i + 3)}{2}$\;
			$I_i \leftarrow [\mu'_i - c_i, \mu'_i + c_i] \cap [0, 1]$\;
		}
		\Comment{\small{Find arm with highest upper confidence bound}}
		$i^* \leftarrow \argmax_i \{ \mu'_i + c_i \} $	\;
		Active set $B \leftarrow $ all arms whose intervals are chained to $I_{i^*}$\;
		Sample a random variable $E \sim$  Bernoulli($T^{\frac{-\alpha}{2}}$)\;
		\If (Explore) {$E = 1$} {
			Sample an arm $i^t$ uniformly from $[k]$\;
		}
		\Else (Exploit) {
			Sample an arm $i^t$ uniformly from $B$\;
		}
		Choose arm $i^t$ and append observed reward $r_t$ to sample sequence $S_{i^t}$\;
	}
 \label{alg}
\end{algorithm}

\begin{theorem}
The \textsc{Fair-UCB}e algorithm is $\delta$-fair, as defined in Definition \ref{def_fairness}, for $\delta \geq 2 T^{-\frac{\alpha}{2}}$.
\label{thm:fair}
\end{theorem}

\begin{theorem}
The regret $R(T)$ achieved by \textsc{Fair-UCB}e satisfies
\begin{displaymath}
R(T) \leq O  \left(k^{3/2}  T^{1 - \frac{\alpha}{2}} \sqrt{\log \left( \sqrt{\frac{k\pi^2}{6}} T^{1 + \frac{\alpha}{4}} \right) } \right).
\end{displaymath}
\label{thm:regret}
\end{theorem}

\begin{remark}
The above regret bound is non-trivial since it guarantees that this fair algorithm achieves sublinear regret even in the context of a non-stationary environment. 
\end{remark}

\section{Discussion}
\label{sec_discussion}
The choice of parameters $\alpha$ and $\epsilon$ in the algorithm is constrained by the inequality $\kappa > 2 \alpha + \epsilon$ or equivalently $\alpha < \frac{1}{2} (\kappa - \epsilon)$. For large $T$, $\epsilon$ can be chosen close to $0$, leading to the constraint $\kappa > 2 \alpha$. Thus, when the non-stationarity in the environment is very high and $\kappa \rightarrow 0$, we have $\alpha \rightarrow 0$ as well. Similary, when the non-stationarity is very low and the environment is almost stationary, $\kappa$ is large and $\alpha$ can be chosen close to $1$.

\cite{joseph2016fairness} showed that in a stationary environment, their $\delta$-fair algorithm \textsc{FairBandits} achieves a regret of $O\left( \sqrt{k^3 T \log \left( \frac{Tk}{\delta} \right)} \right)$, for $\delta < \frac{1}{\sqrt{T}}$, which is $O\left( \sqrt{k^3 T \log \left( k T^{\frac{3}{2}} \right)} \right)$ in the limiting case of $\delta \rightarrow \frac{1}{\sqrt{T}}$. They also showed that \textsc{FairBandits} achieves the best possible performance in that setting. One can see that this bound is equivalent, upto logarithmic factors, to the regret of \\ $O\left( \sqrt{k^3 T \log \left( \sqrt{\frac{k\pi^2}{6}} T^{\frac{5}{4}} \right) } \right)$ achieved by \textsc{Fair-UCB}e in the limiting case of $\alpha \rightarrow 1$, which occurs for large $\kappa$ and $T$. In other words, as the non-stationarity of the environment is reduced, the performance of our algorithm remains consistent with the best performance possible in that setting.

When the change in the environment is high (i.e., $\kappa$ is close to zero), the regret bound for \textsc{Fair-UCB}e is similar to that of SW-UCB\# (Sliding Window Upper Confidence Bound) \citep{wei2018abruptly}, which assumes a non-stationarity constraint similar to ours but does not maintain fairness. The regret bounds are almost the same in terms of $T$ up to logarithmic factors, with the difference being a factor of $T^{\epsilon/4}$, whose exponent goes to zero for large $T$.

\subsection{On Exploration}
One aspect of our algorithm that distinguishes it from other upper confidence bound algorithms is the incorporation of an explicit fixed probability of exploration, in addition to the implicit exploration present in other similar algorithms. This exploration probability $T^{-\frac{\alpha}{2}}$ depends on the non-stationarity of the environment through the constraint $\alpha < \frac{1}{2}(\kappa - \epsilon)$. Smaller values of kappa lead to a larger probability of exploration. This is intuitive in the sense that the more the environment varies,
the greater the need to sample inactive arms via explicit exploration to keep track of changes in their reward distributions. Thus, there is a smooth trade-off between exploitation and exploration, depending on the degree of non-stationarity of the environment.

\subsection{On Sublinearity}
Even though the upper bound $O \left(k^{\frac{3}{2}}  T^{1 - \frac{\alpha}{2}} \sqrt{\log (CT)} \right)$ for the regret of the algorithm is seemingly sublinear in $T$ (since the exponent of $T$ is $1 - \frac{\alpha}{2} < 1$), the extra factor of $k^\frac{3}{2}$ may actually result in the regret being more than $T$. In order to achieve sublinear regret, ignoring logarithmic factors for simplicity, it is necessary that $k^{\frac{3}{2}} T^{1 - \frac{\alpha}{2}} = O (T)$, or equivalently $T = \Omega (k^{\frac{3}{\alpha}})$. Due to the constraint $\alpha < \frac{1}{2}(\kappa - \epsilon)$, $k^\frac{3}{\alpha}$ increases drastically for small values of $\kappa$, necessitating a very large value of $T$ to obtain sublinear regret. In other words, the regret of the algorithm is linear in the context of highly non-stationary environments. As the non-stationarity reduces, $\alpha \rightarrow 1$ and the constraint becomes $T = \Omega (k^3)$, which is identical to the constraint for \textsc{FairBandits} \citep{joseph2016fairness}.

\section{Proofs for Theorems \ref{thm:fair} and \ref{thm:regret}}
\label{analysis}

\subsection{Proof of Theorem \ref{thm:fair}}
At each time step, all arms in the active set are assigned an equal probability of being chosen, say $p_{\mathrm{act}}$, and all arms not in the active set are assigned an equal probability of being chosen, say $p_{\mathrm{nonact}}$. Now, $p_{\mathrm{nonact}} = \frac{1}{kT^{\alpha/2}}$, since $\frac{1}{T^{\alpha/2}}$ is the probability of exploration and $\frac{1}{k}$ is the probability of choosing a specific arm when exploring. For any arm in the active set, the probability $p_{\mathrm{act}}$ of being chosen is 
\begin{displaymath}
\left( 1 - \frac{1}{T^{\frac{\alpha}{2}}} \right) \frac{1}{\# \text{ of active arms}} + \frac{1}{kT^{\frac{\alpha}{2}}} \geq \frac{1}{k} \geq \frac{1}{kT^{\alpha/2}}.
\end{displaymath}

From this, we have $p_{\mathrm{act}} \gneq p_{\mathrm{nonact}}$. Therefore, due to the choice of the active set definition, for the algorithm to be $\delta$-fair, it is sufficient to prove that with probability at least $1 - \delta$, the expected rewards of all arms at all time steps fall in their confidence intervals. Now we proceed to prove these results. 

\subsubsection{Spread of samples}
For any arm $i \in [k]$ at any time $t \in [T]$, the probability of that arm being chosen is at least $\frac{1}{kT^{\alpha/2}}$. So, the expected number of time steps required for obtaining at least one sample from an arm is at most $kT^{\frac{\alpha}{2}}$. This fact can be used to prove the following Lemma.

\begin{lemma}
Let the time interval $[0,T]$ be divided into intervals of size $kT^{\frac{\alpha}{2} + \epsilon}$, for $\epsilon \in (0, 1)$ as specified in Algorithm \ref{alg}. Let $G$ be the event that each arm has atleast one sample in each of these intervals. Then, $\mathsf{P}(G) \geq 1 - \delta_1$, for $\delta_1 = \frac{1}{T^{\frac{\alpha}{2}}}$.
\label{lemma:spread}
\end{lemma}

The constraint on $\epsilon$ can be simplified by the following Lemma. The proofs of both these Lemmas are given in the Appendix.

\begin{lemma}
$$\frac{1}{\log T} \log \left( \frac{1}{2\log (\frac{18}{11})} \log T \right) \leq \frac{1}{2 e \log(18/11)}.$$
\label{lemma:eps}
\end{lemma}

From the above Lemma, $\epsilon = 0.3735$ is sufficient for arbitrary $T$. Moreover, for $T > 15$, the lower bound is a decreasing function of $T$, and thus $\epsilon$ can be chosen much smaller, with the value going to $0$ for large $T$.

\subsubsection{Sufficiency of samples}
At each point $t$ in time, we wish to choose the latest $t^\alpha / k$ samples. But these many samples may not be available, especially if $t$ is small. From Lemma \ref{lemma:spread}, we see that if $G$ is true, then each sample requires atmost $kT^{\frac{\alpha}{2} + \epsilon}$ time steps. So, for the availability of sufficient number of samples, we need $T^{\frac{\alpha}{2} + \epsilon}  t^\alpha < t$. Suppose $M = T^{\frac{1}{1 - \alpha} (\frac{\alpha}{2} + \epsilon)}$, then for $t > M$, we have $t > T^{\frac{1}{1 - \alpha} (\frac{\alpha}{2} + \epsilon)}$, which implies $t^{1 - \alpha} > T^{\frac{\alpha}{2} + \epsilon}$ and hence $T^{\frac{\alpha}{2} + \epsilon}  t^\alpha < t$.

So, if $G$ is true, after the initial $M$ time steps, the number of samples is always sufficient to construct a good confidence interval, provided the exponent of $T$ is sensible. We add the constraint, $\frac{1}{1 - \alpha} (\frac{\alpha}{2} + \epsilon) < 1 - \frac{\alpha}{2}$, to be satisfied by the exponent. This will be useful in the regret calculation. This can be simplified to $\epsilon < \frac{(\alpha - 2)^2}{2} - 1$ or $\alpha < 2 - \sqrt{2\epsilon + 1}$, which is the constraint specified in Algorithm \ref{alg}.

\subsubsection{Confidence intervals}
At time $t$, for arm $i$, consider the latest $\tau_i(t)$ rewards obtained. Let those rewards be $X_i^{(t')}, t' \in [\tau_i(t)]$, where $X_i^{(t')} \in [0,1]$  with means $\mu_i^{(t')}$. The Hoeffding inequality gives
\begin{align*}
	\mathsf{P}\left\{ \left| \sum_{t'} \left( X_i^{(t')} - \mu_i^{(t')} \right) \right| \geq A \right\} \leq 2 \exp \left( \frac{-2 A^2}{\tau_i(t)} \right).
\end{align*}

We make use of the following Lemma, whose proof is provided in the Appendix.

\begin{lemma}
Let $\hat{\mu}_{i_t} = \frac{1}{\tau_i(t)} \sum_{t'} X_i^{(t')}$, the empirical estimate of the expected reward of arm $i$ at time $t$, using the latest $\tau_i(t)$ samples obtained from that arm. Then, if $G$ is true, for $c_i = \frac{A}{\tau_i(t)} + kT^{\frac{\alpha}{2} + \epsilon - \kappa} \frac{(\tau_i(t) + 3)}{2} $,
\begin{align*}
	\mathsf{P} \left\{ \mu_{i_t} \notin \left( \hat{\mu}_{i_t} - c_i, \hat{\mu}_{i_t} + c_i \right) \right\} \leq 2 \exp \left( \frac{-2 A^2}{\tau_i(t)} \right).
\end{align*}
\label{lemma:confidence}
\end{lemma}
By letting $A = \sqrt{\frac{\tau_i(t)}{2} \log( \frac{k\pi^2 t^2}{3 \delta_2})}$, the above probability (that the true mean is outside the confidence interval) summed over all arms $i$ and times $t \in [T]$ is bounded above by
\begin{align*}
\sum_i \sum_t 2 \exp \left( \frac{-2 A^2}{\tau_i(t)} \right) & \leq \sum_i \sum_t 2 \exp \left( - \log \frac{k\pi^2 t^2}{3 \delta_2} \right) \\
& = \sum_{i \in [k]} \sum_t \frac{6 \delta}{k \pi^2 t^2} \leq \delta_2.
\end{align*}
Here, we use the fact that $\sum_t \frac{1}{t^2} \leq \frac{\pi^2}{6}$. Since the above analysis holds if $G$ is true, which happens with probability atleast $1 - \delta_1$, the above confidence intervals hold with probability atleast $1 - \delta_1 - \delta_2$ and the algorithm becomes $\delta$-fair, where $\delta = \delta_1 + \delta_2$.

\subsection{Proof of Theorem \ref{thm:regret}}
The length $\eta_i(t)$ of the confidence interval of arm $i$ is
\begin{displaymath}
\eta_i(t) = 2 \sqrt{\frac{1}{2 \tau_i(t)} \log \left( \frac{k\pi^2 t^2}{3 \delta} \right) } + 2 kT^{\frac{\alpha}{2} + \epsilon - \kappa} \frac{ \left( \tau_i(t) + 3 \right) }{2}.
\end{displaymath}

The regret at any time step of exploitation is atmost $k$ times the size of the largest confidence interval, and atmost $1$ (and also, with probability less than $\delta$, when any of the means fall outside their confidence intervals, it will be bounded by $1$). When $G$ is true, for the first $M$ time steps, the number of samples may be insufficient and the regret is bounded by $1$, and for later time steps, $\tau(t) = \tau_i(t) = \frac{t^\alpha}{k}$, $\eta(t) = \eta_i(t)$ for all $i \in [k]$ and
\begin{align*}
R(T) &\leq \sum_{t \leq M} 1 + \sum_{t > M} k \eta(t) + \sum_t \frac{1}{T^{\frac{\alpha}{2}}} 1 + \delta T \\
	&\leq M + \sum_{t > M} k \eta(t) + T^{1 - \frac{\alpha}{2}} + \delta T,
\end{align*}
where the first two terms are due to exploitation epochs, the third term is due to exploration epochs and the fourth term corresponds to the event $G^c$. Since $M = T^{\frac{1}{1 - \alpha} (\frac{\alpha}{2} + \epsilon)}$, we obtain

$$
R(T) \leq T^{\frac{1}{1 - \alpha} (\frac{\alpha}{2} + \epsilon)} + \sum_{t > M} k \eta(t) + \delta T + T^{1 - \frac{\alpha}{2}}.
$$

Since $\alpha < 2 - \sqrt{2 \epsilon + 1}$ and equivalently $\frac{1}{1 - \alpha} (\frac{\alpha}{2} + \epsilon) < 1 - \frac{\alpha}{2}$, the first term becomes $O( T^{ 1 - \frac{\alpha}{2} })$.

Now, consider the second term in the regret bound. Let $C = \pi \sqrt{\frac{k}{3 \delta}}$. When $G$ is true and $t > M$,  $\tau_i(t) = \frac{t^\alpha}{k}$ and the term becomes

\begin{align*}
& k \sum_{t > M} \left( 2 \sqrt{\frac{1}{2 \tau(t)} \log \left( \frac{k\pi^2 t^2}{3 \delta} \right) } + kT^{\frac{\alpha}{2} + \epsilon - \kappa} \left( \tau(t) + 3 \right) \right) \\
& \leq 2k \int_1^T \sqrt{\frac{k}{2 t^\alpha} \log \left( \frac{k\pi^2 t^2}{3 \delta} \right) } \; \ud t + 4k^2 T^{\frac{\alpha}{2} + \epsilon - \kappa} \sum_{t=1}^T \frac{t^\alpha}{k} \\
& = 2k \sqrt{k}  \int_1^T \sqrt{t^{-\alpha} \log \left( Ct \right) } \; \ud t + 4k T^{\frac{\alpha}{2} + \epsilon - \kappa} \frac{T^{\alpha + 1}}{\alpha + 1}.
\end{align*}

The second term in the above expression is $O(kT^{\frac{\alpha}{2} + \epsilon - \kappa + \alpha + 1})$, which is $O(kT^{1 - \frac{\alpha}{2}})$, since $\kappa > 2 \alpha + \epsilon$. By letting $\log (Ct) = x$, the first term in the above expression becomes
\begin{displaymath}
2k \sqrt{k} \int_{\log C}^{\log (CT)} \sqrt{\frac{1}{C^{-\alpha}} e^{-\alpha x} x} \; \frac{e^x}{C} \; \ud x \\
= \frac{2k \sqrt{k}}{C^{1 - \alpha/2}} \int_{\log C}^{\log (CT)}  e^{x \left(1 - \frac{\alpha}{2} \right) } \sqrt{x} \; \ud x.
\end{displaymath}

Let $\sqrt{x} = t$, then $\frac{1}{2 \sqrt{x}} \ud x = \ud t \implies \sqrt{x} \ud x = 2x \ud t$ and $x = t^2$. The above expression becomes

\begin{align*}
\frac{4k \sqrt{k}}{C^{1 - \alpha/2}} \int_{\sqrt{\log C}}^{\sqrt{\log (CT)}}  e^{t^2(1 - \frac{\alpha}{2})} t^2 \; \ud t = & \frac{4k \sqrt{k}}{C^{1 - \alpha/2}} \frac{1}{1 - \frac{\alpha}{2}} \int_{\sqrt{(1 - \frac{\alpha}{2}) \log C}}^{\sqrt{(1 - \frac{\alpha}{2}) \log (CT)}}  e^{t^2} t^2 \; \ud t \\
\leq & \frac{4k \sqrt{k}}{C^{1 - \alpha/2}} \frac{1}{1 - \frac{\alpha}{2}} \frac{t e^{t^2}}{2} |_{\sqrt{(1 - \frac{\alpha}{2}) \log C}}^{\sqrt{(1 - \frac{\alpha}{2}) \log (CT)}} \\
\leq & \frac{2k \sqrt{k}}{C^{1 - \alpha/2}} \frac{1}{1 - \frac{\alpha}{2}} \sqrt{ \left( 1 - \frac{\alpha}{2} \right) \log (CT)} \; e^{(1 - \frac{\alpha}{2}) \log (CT)} \\
= & \frac{2k \sqrt{k}}{\sqrt{1 - \frac{\alpha}{2}}}  \; T^{1 - \frac{\alpha}{2}} \sqrt{\log (CT)}.
\end{align*}

The first inequality above is obtained by using integration by parts with functions $t$ and $t e^{t^2}$. The overall regret bound becomes
\begin{align*}
R(T) \leq O & \left(k^{\frac{3}{2}}  T^{1 - \frac{\alpha}{2}} \sqrt{\log (CT)} \right) + \delta T.
\end{align*}

Now, $\delta_1 = \frac{1}{T^{\frac{\alpha}{2}}}$. So, letting $\delta_2 = \frac{1}{T^{\frac{\alpha}{2}}}$, we have $\delta = \delta_1 + \delta_2 = O(\frac{1}{T^{\frac{\alpha}{2}}})$ and hence $\delta T = O(T^{1 - \frac{\alpha}{2}})$. Therefore,
\begin{align*}
R(T) & \leq O  \left(k^{\frac{3}{2}}  T^{1 - \frac{\alpha}{2}} \sqrt{\log \left( \sqrt{\frac{k\pi^2}{6}} T^{1 + \frac{\alpha}{4}} \right) } \right), \: \text{provided} \\
\kappa & > 2 \alpha + \epsilon, \\
\epsilon & > \frac{1}{\log T} \log \left( \frac{1}{2\log (\frac{18}{11})} \log T \right),\:\:\:\mbox{and} \\
\alpha & < 2 - \sqrt{2\epsilon + 1}.
\end{align*}

\section{Experiments}
\label{experiments}
In this section, we present results from applying the proposed algorithm in a simulated environment. We consider a bandit with $k=10$ arms. The initial expected rewards for all arms are chosen uniformly from $[0.05, 0.95]$. The rewards at each time step are sampled from a beta distribution. For non-stationarity, each arm is assigned randomly at the beginning to be drifting either upwards or downwards. An upward drifting arm is more likely to drift upwards with some fixed probability $0.8$ and vice versa. At each time step, the drift in the expected value of each arm is sampled uniformly from $[0, T^{-\kappa}]$ and added to the expected reward while also constraining it to remain within the original interval. The results are plotted in Figure \ref{fig_plot1} as the ratio between the regret $R(t)$ achieved by the algorithm and the regret bound $B(T) = k^{\frac{3}{2}} T^{1 - \frac{\alpha}{2}} \sqrt{\log \left( \frac{1}{\log \left( \frac{18}{11} \right) } T^{1 + \frac{\alpha}{4}} \right) }$.

\begin{figure*}
\centering

\begin{subfigure}{0.32\textwidth}
\centering
\includegraphics[width=\textwidth]{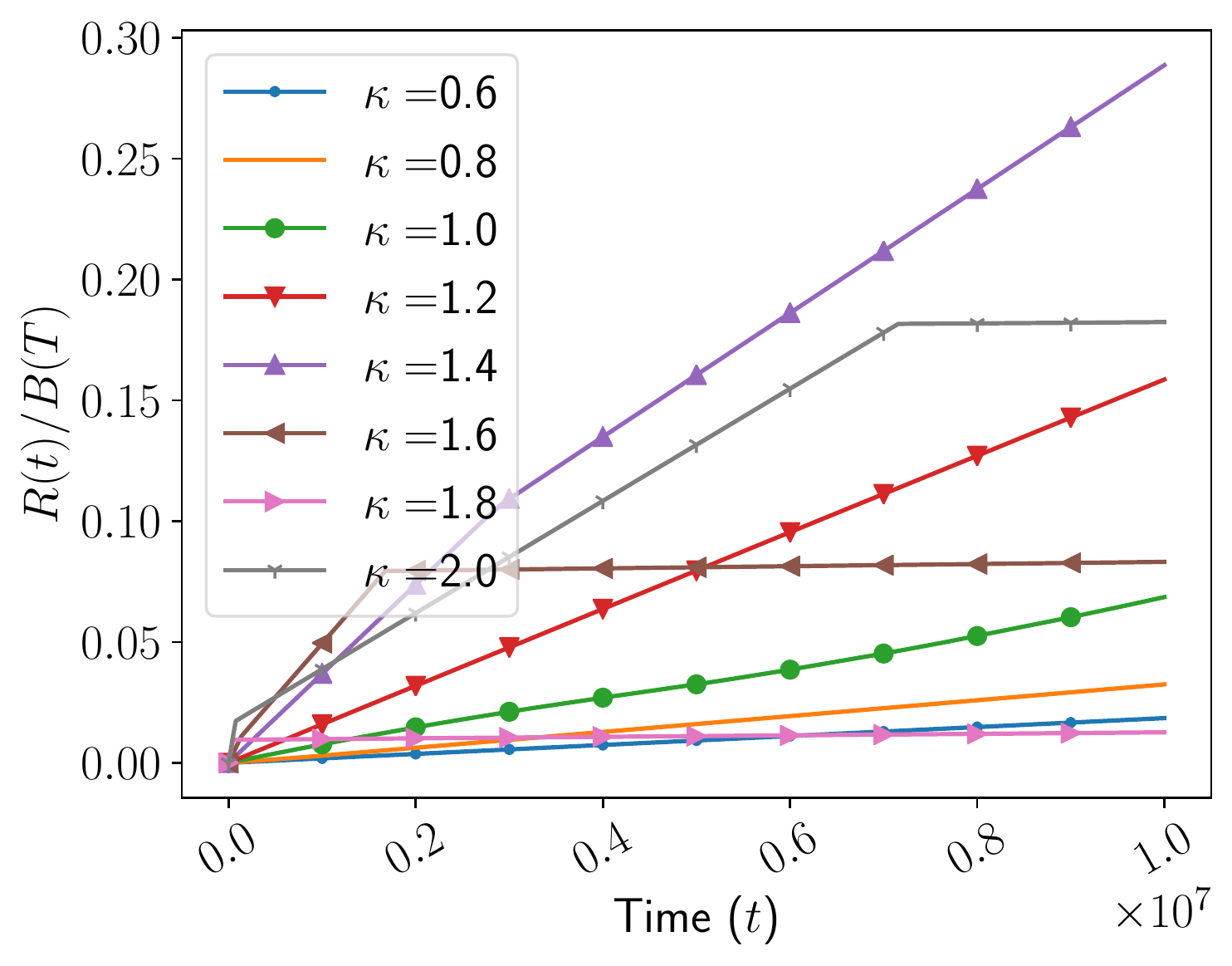}
\caption{}
\label{fig_plot1}
\end{subfigure}
\begin{subfigure}{0.32\textwidth}
\centering
\includegraphics[width=\textwidth]{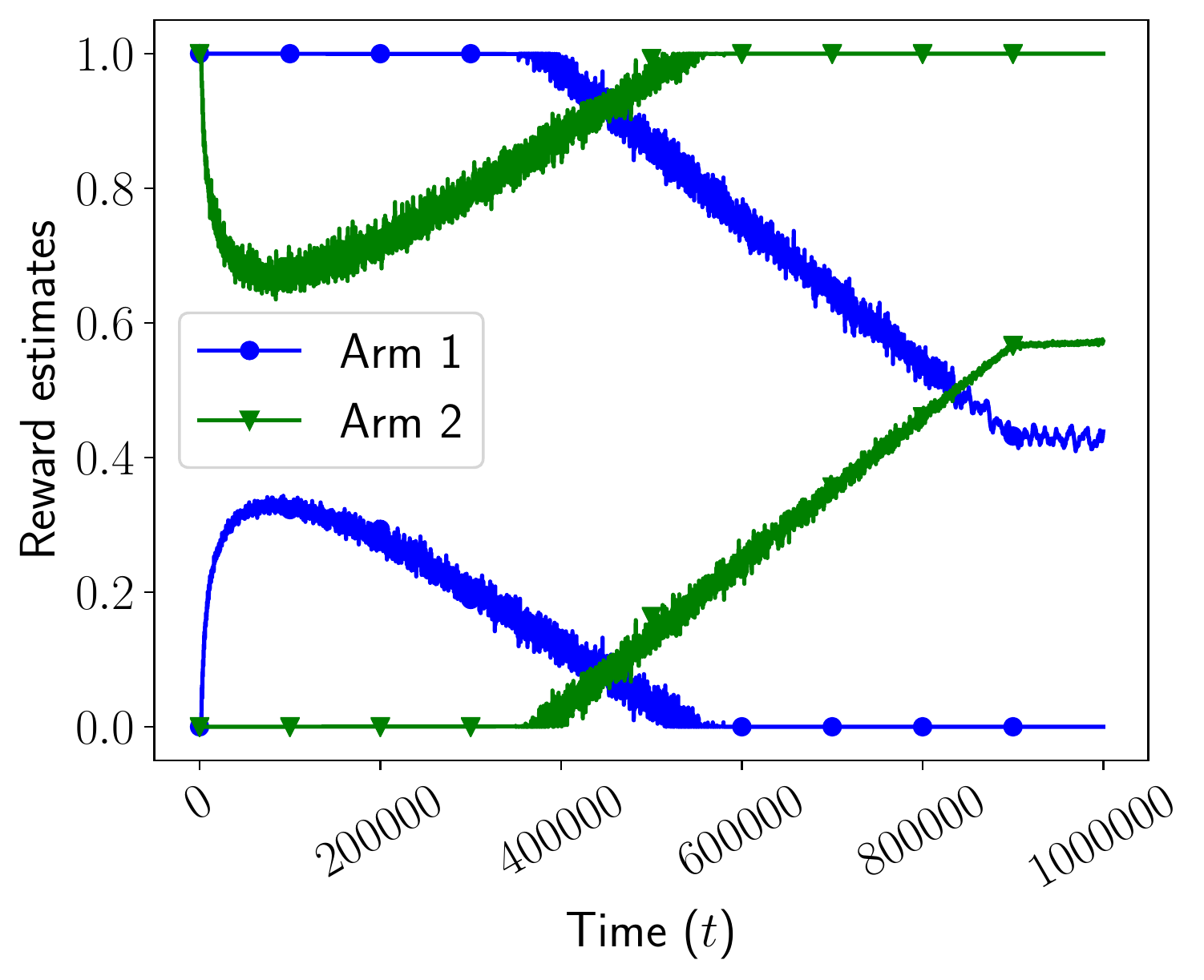}
\caption{}
\label{fig_plot2}
\end{subfigure}
\begin{subfigure}{0.32\textwidth}
\centering
\includegraphics[width=\textwidth]{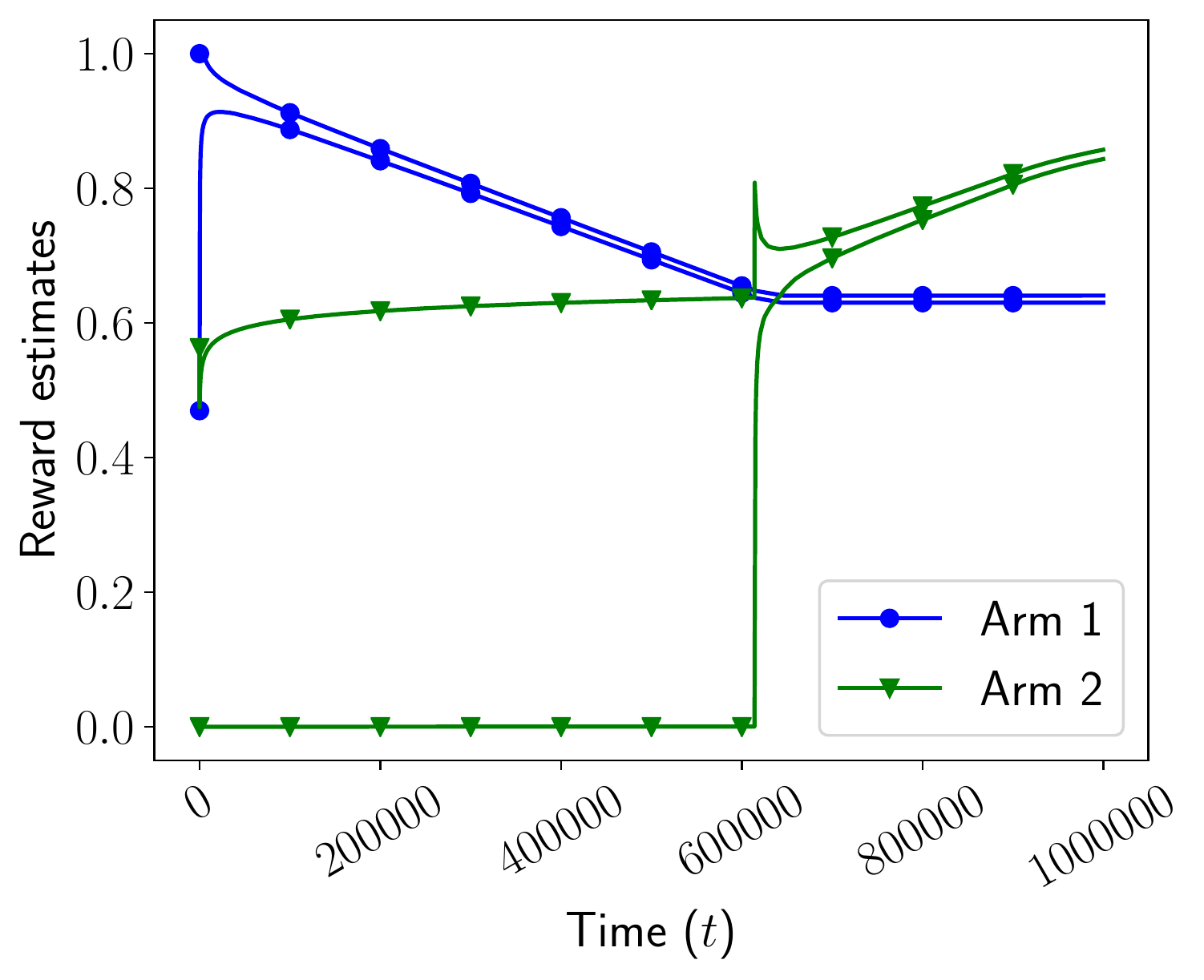}
\caption{}
\label{fig_plot3}
\end{subfigure}
\caption{(a) Plot showing the ratio of regret $R(t)$ achieved by the algorithm and the regret bound $B(T)$, for time horizon $T = 10^7$ and multiple values of $\kappa$. A smaller value of $\kappa$ indicates a fast changing environment. (Best viewed in colour.) (b) and (c) show the change in the confidence intervals of the two arms under \textsc{Fair-UCB}e and \textsc{FairBandits} respectively in an environment with two arms, time horizon $T = 10^6$ and $\kappa = 1.0$, where the expected rewards of the arms continuously evolve in opposite directions. It can be observed that due to lack of exploration, throughout most of the time horizon, \textsc{FairBandits} fails to accurately estimate the reward distribution of whichever arm is not in the active set at that point in time and thus does not maintain fairness.}
\end{figure*}

The apparent linear increase in the regret with time is due to the constant exploration probability at each time step, which is necessary to ensure fairness. This does not contradict the sublinear regret bound since the bound is for the cumulative expected regret achieved over the entire time horizon and does not constrain how the regret changes with time. It can be seen from the Figure \ref{fig_plot1} that the cumulative regret achieved does not exceed the derived upper bound.

The sharp changes in the slopes of some of the lines in the plot correspond to changes in the composition of the active set. Dropping a sub-optimal arm from the active set results in a drastic reduction in the expected regret at each time step of exploitation.

To further illustrate the necessity of a sliding window and exploration to deal with non-stationarity, we consider another experiment with two arms and parameters $T = 10^6$ and $\kappa = 1.0$. In this experiment, we let the first arm start with an expected reward of $0.95$ and decrease by $T^{-\kappa},$ the maximum amount possible, at each time step. Similarly, we let the second arm start with an expected reward of $0.05$ and increase it by $T^{-\kappa}$ at each time step. We run \textsc{Fair-UCB}e and \textsc{FairBandits} in this environment and plot the upper and lower confidence bounds of the arms in Figures \ref{fig_plot2} and \ref{fig_plot3} respectively.

Under \textsc{Fair-UCB}e, we can see the gradual shift of the confidence intervals of both arms as the underlying reward distributions change. In contrast, under \textsc{FairBandits}, we observe that as soon as the first arm becomes known to be better than the second arm, the latter is discarded from the active set, and the algorithm loses track of its reward distribution. Only much later does the estimated mean reward of the first arm become low enough for its confidence interval to overlap with that of the second arm. The algorithm then becomes aware of the change in the reward distribution of the second arm. After a few time steps, the active set again contains only one arm, this time the second arm, and the lack of exploration leads to a biased estimation of the other arm again, as seen by the lack of change in the confidence bounds of the first arm nearing the end of the horizon. Thus, we see that every aspect of our algorithm is crucial for dealing with a non-stationary environment while being fair to all arms.

\section{Related work}
\label{sec_related}
In this work, we considered a specific version of a slowly varying environment to study regret bounds for a fair algorithm solving a non-stationary multi-armed bandit problem. Several variations of non-stationary bandit problems have been extensively studied, which constrain in various ways how the reward distributions of the arms change as time passes. \citet{garivier2011upper} studied a setting with a constraint on the number of times of occurrence of an arbitrary change in the reward distribution, referred to as an abruptly changing environment. Alternatively, change in the reward distribution could be allowed at every time instant, but the nature of each change is constrained, leading to a slowly varying or drifting environment like the one considered in this work. The constraint could be on the absolute change in the expected reward at each time step~\citep{wei2018abruptly}, or a stochastic change in the form of a known distribution~\citep{slivkins2008adapting}. Another extensively studied setting consists of a constraint on the total absolute variation, throughout all time steps, of the expected rewards~\citep{besbes2019optimal, russac2019weighted}.

Our non-stationarity assumption is similar to that of \citet{wei2018abruptly}. Their algorithm, SW-UCB\# (Sliding Window UCB), was designed for dealing with a slowly varying environment in which the change in the expected mean of each arm's reward distribution is assumed to be $O(T^{-\kappa})$, which is a slightly weaker assumption than ours. Its use of an increasing sliding window of samples to estimate the current mean is similar to our algorithm. At each step, instead of considering all the reward samples obtained till then, the only rewards used are those obtained in the last $\lambda t^\alpha$ time steps, for suitable values of $\alpha$ and $\lambda$. However, their work differs from ours in terms of fairness due to the lack of explicit exploration and the arm's deterministic selection with the highest upper confidence bound at every time step. Another difference is that they use a sliding window of a certain size, whereas, in our method, a certain number of latest samples are used, irrespective of how old those samples are.

A similar algorithmic choice for dealing with non-stationarity is the use of a fixed-size sliding window. Another technique is to discount older rewards when estimating the expected reward of an arm. This ensures that older samples affect the estimation less and reduce bias in the estimation induced due to the environment's non-stationarity. These two approaches have been studied by \citet{garivier2011upper} for abruptly varying environments.

EXP3~\cite{auer2002nonstochastic} is an algorithm used for solving an adversarial multi-armed bandit. \citet{besbes2019optimal} repurposed this algorithm and showed that by restarting it every $\Delta$ time steps, for some suitable $\Delta$, this algorithm can be used to solve a stochastic and non-stationary bandit and achieves the lower bound of regret achievable in that setting.

The notion of $\delta$-fairness considered in this paper has been studied for classic contextual bandits in \cite{joseph2016fairness}. Another notion of fairness is to constrain the fraction of times an arm is chosen by a pre-specified lower bound \citep{li2019combinatorial}. However, the $\delta$-fairness of \cite{joseph2016fairness} differs from this notion significantly since this definition of fairness depends on an external lower bound specification independent of the reward distributions of the arms themselves.

\section{Conclusion}
\label{sec_conclusion}
In this work, we have studied the problem of designing a $\delta$-fair algorithm for a stochastic non-stationary multi armed bandit problem. Our non-stationarity assumption is that the absolute change in the expected reward of each arm is assumed to be at most $T^{-\kappa}$ at each time step, for some known $\kappa \in \mathbb{R}_+$. We have shown that the proposed algorithm \textsc{Fair-UCB}e indeed satisfies $\delta$-fairness condition, for $\delta \geq 2 T^{\frac{-\alpha}{2}}$. We also show that it 
achieves a regret of $O \big(k^{\frac{3}{2}}  T^{1 - \frac{\alpha}{2}} \sqrt{\log c T^{1 + \frac{\alpha}{4}})} \big)$, for some constant $c \in \mathbb{R}_+$.

\appendix
\section{Proof of Lemma~\ref{lemma:spread}-\ref{lemma:confidence}}

For the proof of Lemma \ref{lemma:spread}, we need the following result.
\begin{lemma}
For $x \in [0, \frac{1}{2}]$, $n \geq 1$,
$
\left( 1 - \frac{x}{n} \right)^n \leq 1 - \frac{7x}{9}.
$
\label{lemma:x}
\end{lemma}
\begin{proof}
For $n \geq 1$, $f_n(x) = ( 1 - \frac{x}{n} )^n$ is a convex function in $[0, \frac{1}{2}]$. For $n = 1$, this is clear since it is linear, and for $n = 2$, the function becomes $(\frac{x}{2} - 1)^2$, which is clearly a convex quadratic function. For $n > 2$, $f'_n(x)=-(1-\frac{x}{n})^{n - 1}$ and $f''_n(x) = \frac{n-1}{n} (1 - \frac{x}{n})^{n-2} \geq 0$.

Now, consider the sequence $f_n(1/2) = ( 1 - \frac{1}{2n} )^n$. Clearly, it is an increasing sequence, and it converges to $1 / \sqrt{e}$, which implies $f_n(1/2) \leq 1 / \sqrt{e}$.

So, for $x \in [0, \frac{1}{2}]$, $x = (1 - 2x).0 + 2x.\frac{1}{2}$, and by the convexity of $f_n$,

\begin{align*}
f(x) & \leq (1 - 2x). f(0) + 2x. f \bigg( \frac{1}{2} \bigg) \\
& \leq 1 - 2x + 2x. \frac{1}{\sqrt{e}} \\
& = 1 - x \bigg( 2 - \frac{2}{\sqrt{e}} \bigg) \\
& \leq 1 - \frac{7x}{9},
\end{align*}
since $2 - \frac{2}{\sqrt{e}} = 0.7869...$ and $\frac{7}{9} = 0.777..$.
\end{proof}

\subsection{Proof of Lemma~\ref{lemma:spread}}
The probability that more than $N$ time steps are required to obtain a single sample is at most
$\left( 1 - \frac{1}{kT^{\frac{\alpha}{2}}}  \right)^N$, for an arm to be rejected consecutively at least $N$ times. This is the probability that, for a specific arm, there is no sample in a single time interval of length $N$. We need to consider this failure probability for all intervals and all arms. So, setting interval size $N = kT^{\frac{\alpha}{2} + \epsilon}$ and dividing the available failure probability $\delta$ into $k$ parts for each of the arms, we need
\begin{displaymath}
\frac{T}{kT^{\frac{\alpha}{2} + \epsilon}} \left( 1 - \frac{1}{kT^{\frac{\alpha}{2}}} \right)^{kT^{\frac{\alpha}{2} + \epsilon}} < \frac{\delta_1}{k},
\end{displaymath}
or equivalently,
\begin{displaymath}
\left( 1 - \frac{1}{kT^{\frac{\alpha}{2}}} \right)^{kT^{\frac{\alpha}{2} + \epsilon}} < \delta_1  T^{\frac{\alpha}{2} + \epsilon - 1}  = T^{\epsilon - 1}.
\end{displaymath}
By Lemma \ref{lemma:x}, with $x = \frac{1}{2}$ and $n = \frac{kT^{\frac{\alpha}{2}}}{2}$,
\begin{align*}
\left( 1 - \frac{1}{kT^{\frac{\alpha}{2}}} \right)^{kT^{\frac{\alpha}{2} + \epsilon}} & = \left( 1 - \frac{1/2}{kT^{\frac{\alpha}{2}}/2} \right)^{\frac{kT^{\frac{\alpha}{2}}}{2}2T^\epsilon} \\
& \leq \left( 1 - \frac{7}{18} \right)^{2 T^\epsilon}.
\end{align*}

So, for the Lemma to hold, $\epsilon$ should satisfy $\big( \frac{11}{18} \big)^{2 T^\epsilon} < T^{\epsilon - 1}$. From Algorithm 1, we have
\begin{align*}
\epsilon & > \frac{1}{\log T} \log \left( \frac{1}{2\log (\frac{18}{11})} \log T \right),
\end{align*}
which implies $T^\epsilon  > \frac{1}{2\log (\frac{18}{11})} \log T > \frac{\epsilon - 1}{2\log (\frac{11}{18})} \log T$, and hence $\left( \frac{11}{18} \right)^{2 T^\epsilon} < T^{\epsilon - 1}$.

\subsection{Proof of Lemma~\ref{lemma:eps}}
Consider the function $f(x) = \frac{\log x}{x}$. $f'(x) = \frac{1 - \log x}{x^2}$, which is positive for $x < e$ and negative for $x > e$. So, $f(x)$ attains its maximum at $x = e$, with $f(e) = \frac{1}{e}$. Hence,
\begin{displaymath}
\frac{1}{\log T} \log \left( \frac{1}{2\log (\frac{18}{11})} \log T \right) \\
= \frac{\log \left( \frac{1}{2\log (\frac{18}{11})} \log T \right)}{\frac{1}{2\log (\frac{18}{11})} \log T} \frac{1}{2\log (\frac{18}{11})}
\end{displaymath}
attains a maximum of $\frac{1}{2e \log (\frac{18}{11})}$, at $\frac{1}{2\log (\frac{18}{11})} \log T = e$ or $T = e^{2e \log (\frac{18}{11})}$, which is approximately 14.5669... Since we considered the maximum value possible, this upper bound is a worst case bound for the value of $\epsilon$ and for $T \geq 15$, this can be improved considerably, since the function reduces to $0$ as $T$ increases.

\subsection{Proof of Lemma~\ref{lemma:confidence}}
We have
\begin{align*}
\mathsf{P}\left\{ \left| \sum_{t'} \left( X_i^{(t')} - \mu_i^{(t')} \right) \right| \geq A \right\} \leq 2 \exp \left( \frac{-2 A^2}{\tau_i(t)} \right).
\end{align*}

Let $\mu_i^{(t')} = \mu_i^t + \epsilon_t^{(t')}$, where $\mu_i^t$ is the true mean at time $t$, and the true mean at a previous time differs from this by some error $\epsilon_t^{(t')}$, which is bounded by $T^{-\kappa}$ times the time difference, since $G$ is true. Now,
\begin{align*}
 \left| \sum_{t'} \left( X_i^{(t')} - \mu_i^{(t')} \right) \right| &
 = \left| \sum_{t'} \left( X_i^{(t')} - \mu_i^t - \epsilon_t^{(t')} \right) \right| \\
& \geq \left| \sum_{t'} \left( X_i^{(t')} - \mu_i^t \right) \right| - \left| \sum_{t'} \epsilon_t^{(t')}  \right| \\
& \geq \left| \sum_{t'} \left( X_i^{(t')} - \mu_i^t \right) \right| - \sum_{t'} \left| \epsilon_t^{(t')}  \right|.
\end{align*}
So, for $\left| \sum_{t'} \left( X_i^{(t')} - \mu_i^{(t')} \right) \right| \geq A$, it is sufficient that
\begin{displaymath}
\frac{1}{\tau_i(t)} \left| \sum_{t'} \left( X_i^{(t')} - \mu_i^t \right) \right| \geq \frac{1}{\tau_i(t)} \left( A + \sum_{t'} \left| \epsilon_t^{(t')} \right| \right) . 
\end{displaymath}
Furthermore,
\begin{align*}
\quad \frac{1}{\tau_i(t)} \left( A +  \sum_{t'} \left| \epsilon_t^{(t')} \right| \right) & \leq \frac{1}{\tau_i(t)} \left( A +  \sum_{j=1}^{\tau_i(t)} (j+1) kT^{\frac{\alpha}{2} + \epsilon} T^{-\kappa} \right) \\
& \leq \frac{1}{\tau_i(t)} \left( A +  kT^{\frac{\alpha}{2} + \epsilon - \kappa} \sum_{j=1}^{\tau_i(t)} (j+1) \right) \\
& \leq \frac{1}{\tau_i(t)} \left( A +  kT^{\frac{\alpha}{2} + \epsilon - \kappa} \frac{\tau_i(t)(\tau_i(t) + 3)}{2} \right) \\
& = \frac{A}{\tau_i(t)} + kT^{\frac{\alpha}{2} + \epsilon - \kappa} \frac{(\tau_i(t) + 3)}{2}.
\end{align*}
Therefore, for $\hat{\mu}_i^t = \frac{1}{\tau_i(t)} \sum_{t'} X_i^{(t')}$, the empirical estimate of the expected reward of arm $i$ at time $t$ using the latest $\tau_i(t)$ samples obtained from that arm, and $c_i = \frac{A}{\tau_i(t)} + kT^{\frac{\alpha}{2} + \epsilon - \kappa} \frac{(\tau_i(t) + 3)}{2} $, we have
\begin{align*}
\mathsf{P} \left\{ \mu_i^t \notin \left( \hat{\mu}_i^t - c_i, \hat{\mu}_i^t + c_i \right) \right\} = & \; \; \mathsf{P} \left\{ \frac{1}{\tau_i(t)} \left| \sum_{t'} \left( X_i^{(t')} - \mu_i^t \right) \right| \geq c_i \right\} \\
\leq & \; \; \mathsf{P} \bigg\{ \frac{1}{\tau_i(t)} \left| \sum_{t'} \left( X_i^{(t')} - \mu_i^t \right) \right| \geq \frac{1}{\tau_i(t)} \left( A + \sum_{t'} \left| \epsilon_t^{(t')} \right| \right) \bigg\} \\
\leq & \; \; \mathsf{P} \left\{ \left| \sum_{t'} \left( X_i^{(t')} - \mu_i^{(t')} \right) \right| \geq A \right\} \\
\leq & \; \; 2 \exp \left( \frac{-2 A^2}{\tau_i(t)} \right).
\end{align*}

\footnotesize
\bibliographystyle{chicago}
\bibliography{paper}

\end{document}